\pdfoutput=1

\documentclass[twoside]{article}
\usepackage[accepted]{aistats2024}
\usepackage[utf8]{inputenc} 
\usepackage[T1]{fontenc}    
\usepackage{hyperref}       
\usepackage{url}            
\usepackage{booktabs}       
\usepackage{amsfonts}       
\usepackage{nicefrac}       
\usepackage{microtype}      

\usepackage{microtype}
\usepackage{graphicx}
\usepackage{subcaption}
\usepackage{multirow}
\usepackage{booktabs} 
\usepackage{bm}
\usepackage{nicefrac}
\usepackage{siunitx}
\usepackage{algorithmic}
\usepackage{float}
\usepackage[inline]{enumitem}

\usepackage[svgnames]{xcolor}
\usepackage{hyperref}
\hypersetup{
  hidelinks = false,
  colorlinks   = true, 
  urlcolor     = DarkBlue,
  linkcolor    = DarkBlue, 
  citecolor   = DarkBlue}

\usepackage{amsmath}
\usepackage{mathtools}
\usepackage{amssymb}
\usepackage{amsthm}
\usepackage[capitalize,noabbrev]{cleveref}
\PassOptionsToPackage{finalizecache=false,frozencache=true}{minted}
\usepackage{tikz-network}
\usepackage[noframes,nobars]{ffcode}

\usepackage[round]{natbib}

\usepackage{citeref}
\usepackage[T1]{fontenc}
\newcommand\thefont{\expandafter\string\the\font}

\usepackage[capitalize,noabbrev]{cleveref}

\usepackage{thmtools, thm-restate}

\theoremstyle{plain}
\newtheorem{theorem}{Theorem}[section]

\newtheorem{corollary}[theorem]{Corollary}
\theoremstyle{definition}
\newtheorem{definition}[theorem]{Definition}

\theoremstyle{remark}

\newtheorem{example}[theorem]{Example}

\newcommand{\set}[1]{\mathcal{#1}}
\newcommand{\mat}[1]{\bm{#1}}
\newcommand{\ten}[1]{\bm{\mathcal{#1}}}
\newcommand{\inner}[2]{\left\langle #1, #2 \right\rangle}
\newcommand{\innerfrob}[2]{{\langle #1, #2 \rangle}}

\newcommand{\vectorize}[1]{\operatorname{\ff{vec}}\left(#1\right)}
\newcommand{\tensorize}[1]{\operatorname{\ff{ten}}\left(#1, M_1, M_2, \ldots, M_D\right)}
\newcommand{\tensorizeQuantized}[1]{\operatorname{\ff{ten}}\left(#1, Q,Q,\ldots,Q\right)}
\newcommand{\tensorizeTimes}[1]{\operatorname{\ff{ten}}(#1, \underbrace{M, M, \ldots, M}_{D \ \text{times}})}
\newcommand{\tensorizeQuantizedTimes}[1]{\operatorname{\ff{ten}}(#1, \underbrace{Q,Q,\ldots,Q}_{D K \ \text{times}})}
\newcommand{\kron}{\otimes}

\newcommand{\conj}{^*}

\newcommand{\order}[1]{$\mathcal{O}(#1)$}

\crefname{equation}{equation}{equations}
\Crefname{equation}{Equation}{Equations}
\crefname{theorem}{theorem}{theorem}
\Crefname{theorem}{Theorem}{Theorems}
\crefname{definition}{definition}{definitions}
\crefname{corollary}{corollary}{corollaries}
\Crefname{corollary}{Corollary}{Corollaries}
\crefname{lemma}{lemma}{lemmas}
\Crefname{lemma}{Lemma}{Lemmas}
\crefname{appendix}{appendix}{appendices}
\crefname{section}{section}{sections}
\crefname{figure}{figure}{figures}
\crefname{table}{table}{tables}

\begin{document}

\twocolumn[
\aistatstitle{Quantized Fourier and Polynomial Features for more Expressive Tensor Network Models}
\aistatsauthor{Frederiek Wesel \And Kim Batselier}
\aistatsaddress{Delft Center for Systems and Control\\ Delft University of Technology \And Delft Center for Systems and Control\\ Delft University of Technology} ]

\begin{abstract}
In the context of kernel machines, polynomial and Fourier features are commonly used to provide a nonlinear extension to linear models by mapping the data to a higher-dimensional space. Unless one considers the dual formulation of the learning problem, which renders exact large-scale learning unfeasible, the exponential increase of model parameters in the dimensionality of the data caused by their tensor-product structure prohibits to tackle high-dimensional problems. One of the possible approaches to circumvent this exponential scaling is to exploit the tensor structure present in the features by constraining the model weights to be an underparametrized tensor network.
In this paper we quantize, i.e. further tensorize, polynomial and Fourier features. 
Based on this feature quantization we propose to quantize the associated model weights, yielding quantized models. We show that, for the same number of model parameters, the resulting quantized models have a higher bound on the VC-dimension as opposed to their non-quantized counterparts, at no additional computational cost while learning from identical features.
We verify experimentally how this additional tensorization regularizes the learning problem by prioritizing the most salient features in the data and how it provides models with increased generalization capabilities. We finally benchmark our approach on large regression task, achieving state-of-the-art results on a laptop computer.
\end{abstract}

\section{INTRODUCTION}
In the context of supervised learning, the goal is to estimate a function $f\left(\cdot\right):\mathcal{X}\rightarrow\mathcal{Y}$ given $N$ input-output pairs $\left\{\mat{x}_n,y_n\right\}_{n=1}^N$, where $\mat{x}\in\mathcal{X}$ and $y\in\mathcal{Y}$. Kernel machines accomplish this by lifting the input data into a high-dimensional feature space by means of a \emph{feature map} $\mat{z}\left(\cdot\right):\mathcal{X}\rightarrow\mathcal{H}$ and seeking a linear relationship therein:
\begin{equation}\label{eq:KernelMachine}
     f\left(\mat{x}\right) = \inner{\mat{z}\left( \mat{x} \right)}{\mat{w}}.
\end{equation}
Training such a model involves the minimization of the regularized empirical risk given a convex measure of loss $\ell\left(\cdot,\cdot\right): \mathcal{H} \times \mathcal{Y} \rightarrow \mathbb{R}_{+}$
\begin{equation}\label{eq:EmpiricalRisk}
    R_\text{empirical}\left(\mat{w}\right) = \frac{1}{N}\sum_{n=1}^N \ell \left(\inner{\mat{z}\left(\mat{x}_n\right)}{\mat{w}},y_n\right) + \lambda \left\vert\left\vert  \mat{w} \right\vert \right\vert^2.
\end{equation}
Different choices of loss yield the \emph{primal} formulation of different kernel machines. For example, squared loss results in kernel ridge regression (KRR) \citep{suykens_least_2002}, hinge loss in support vector machines (SVMs) \citep{cortes_support-vector_1995}, and logistic loss yields logistic regression. Different choices of the feature map $\mat{z}$ allow for modeling different nonlinear behaviors in the data. 
In this article we consider tensor-product features
\begin{equation}\label{eq:TensorProductFeatureMap}
    \mat{z}\left(\mat{x}\right) =  \bigotimes_{d=1}^D \mat{v}^{(d)}\left(x_d\right),
\end{equation}
where $\mat{v}^{(d)}(\cdot):\mathbb{C}\rightarrow\mathbb{C}^{M_d}$ is a feature map acting on each element of the $d$-th component $x_d$ of $\mat{x}\in\mathbb{C}^D$. Here $\kron$ denotes the \emph{left} Kronecker product \citep{cichocki_tensor_2016}. This tensor-product structure arises when considering product kernels \citep{shawe-taylor_kernel_2004,hensman_variational_2017,solin_hilbert_2020}, Fourier features \citep{wahls_learning_2014}, when considering B-splines \citep{karagoz_nonlinear_2020} and polynomials \citep{shawe-taylor_kernel_2004}.


Due to the tensor-product structure in~\cref{eq:TensorProductFeatureMap}, $\mat{z}(\cdot)$ maps an input sample $\mat{x} \in \mathbb{C}^D$ into an exponentially large feature vector $\mat{z}(\mat{x}) \in \mathbb{C}^{M_1 M_2 \cdots M_D}$. As a result, the model is also described by an exponential number of weights $\mat{w}$. This exponential scaling in the number of features limits the use of tensor-product features to low-dimensional data or to mappings of very low degree.

Both these computational limitations can be sidestepped entirely by considering the \emph{dual} formulation of the learning problem in \cref{eq:EmpiricalRisk}, requiring to compute the pairwise similarity of all data respectively by means of a kernel function $k(\mat{x},\mat{x'})=\inner{\mat{z}(\mat{x})}{\mat{z}(\mat{x'})}$.
However, the dual formulation requires to instantiate the kernel matrix at a cost of \order{N^2} and to estimate $N$ Lagrange multipliers by solving a (convex) quadratic problem at a cost of at least \order{N^2}, prohibiting to tackle large-scale data (large $N$). To lift these limitations, a multitude of research has focused on finding low-rank approximations of kernels by considering \emph{random} methods such as polynomial sketching \citep{pham_fast_2013,woodruff_sketching_2014,meister_tight_2019} and random features \citep{williams_using_2001,rahimi_random_2007,le_fastfood_2013}, which approximate the feature space with probabilistic approximation guarantees.

One way to take advantage of the existing tensor-product structure in \cref{eq:TensorProductFeatureMap} is by imposing a tensor network \citep{kolda_tensor_2009,sidiropoulos_tensor_2017} constraint on the weights $\mat{w}$. For example, using a polyadic rank-$R$ constraint reduces the storage complexity of the weights from \order{M^D} down to \order{DMR} and enables the development of efficient learning algorithms with a computational complexity of \order{DMR} per gradient descent iteration. This idea has been explored for polynomial \citep{favier_parametric_2009,rendle_factorization_2010,blondel_higher-order_2016,blondel_multi-output_2017,batselier_tensor_2017-1} pure-power-$1$ polynomials \citep{novikov_exponential_2018}, pure-power polynomials of higher degree \citep{chen_parallelized_2018}, B-splines \citep{karagoz_nonlinear_2020}, and Fourier features \citep{wahls_learning_2014,stoudenmire_supervised_2016,efthymiou_tensornetwork_2019,kargas_supervised_2021,cheng_supervised_2021,wesel_large-scale_2021}. 

In this article, we improve on this entire line of research by deriving an exact \emph{quantized} representation \citep{khoromskij_odlog_2011} of pure-power polynomials and Fourier features, exploiting their inherent Vandermonde structure. It is worth noting that \emph{in this paper quantized means further tensorized}, and should not be confused with the practice of working with lower precision floating point numbers. 
By virtue of the derived quantized features, we are able to quantize the model weights. 
We show that compared to their non-quantized counterparts, quantized models can be trained with no additional computational cost, while learning from the same exact features.
Most importantly, for the same number of model parameters the ensuing quantized models are characterized by higher upper bounds on the VC-dimension, which indicates a potential higher expressiveness. While these bounds are in practice not necessarily met, we verify experimentally that:

\begin{enumerate}
    \item Quantized models are indeed characterized by higher expressiveness. This is demonstrated in \cref{sec:generalization}, where we show that in the underparameterized regime quantized models achiever lower test errors than the non-quantized models with identical features and identical total number of model parameters.
    \item This additional structure regularizes the problem by prioritizing the learning of the peaks in the frequency spectrum of the signal (in the case of Fourier features) (\cref{sec:regularization}). In other words, the quantized structure is learning the most salient features in the data first with its limited amount of available model parameters.
    \item Quantized tensor network models can provide state-of-the-art performance on large-scale real-life problems. This is demonstrated in \cref{sec:large-scale}, where we compare the proposed quantized model to both its non-quantized counterpart and other state-of-the-art methods, demonstrating superior generalization performance on a laptop computer.
\end{enumerate}



\section{BACKGROUND}

We denote scalars in both capital and non-capital italics $w,W$, vectors in non-capital bold $\mat{w}$, matrices in capital bold $\mat{W}$ and tensors, also known as higher-order arrays, in capital italic bold font $\ten{W}$. Sets are denoted with calligraphic capital letters, e.g. $\set{S}$. The $m$-th entry of a vector $\mat{w}\in\mathbb{C}^{M}$ is indicated as $w_m$ and the $m_1 m_2\ldots m_D$-th entry of a $D$-dimensional tensor $\ten{W}\in \mathbb{C}^{M_1 \times M_2 \times \cdots \times M_D}$ as $w_{m_1m_2\ldots m_D}$. We denote the complex-conjugate with superscript $\conj$ and $\kron$ denotes the \emph{left} Kronecker product \citep{cichocki_tensor_2016}. We employ zero-based indexing for all tensors.
The Frobenius inner product between two $D$-dimensional tensors $\ten{V},\ten{W}\in\mathbb{C}^{M_1 \times M_2 \times  \cdots \times M_D}$ is defined as
\begin{equation}\label{eq:inner}
    \innerfrob{\ten{V}}{\ten{W}} \coloneqq \sum_{m_1=0}^{M_1-1}\sum_{m_2=0}^{M_2-1}\cdots \sum_{m_D=0}^{M_D-1} v_{m_1 m_2 \ldots m_D}^* w_{m_1 m_2 \ldots m_D}.
\end{equation}

We define the vectorization operator as $\vectorize{\cdot}:\mathbb{C}^{M_1\times M_2 \times \cdots \times M_D} \rightarrow \mathbb{C}^{M_1 M_2 \cdots M_D}$ such that
\begin{equation*}
    \vectorize{\ten{W}}_{m} = w_{m_1m_2\ldots m_D},
\end{equation*}
 with $m = m_1 + \sum_{d=2}^{D} m_d \prod_{k=1}^{d-1} M_k$. Likewise, its inverse, the tensorization operator $\tensorize{\cdot}: \mathbb{C}^{M_1 M_2 \cdots M_D}\rightarrow \mathbb{C}^{M_1\times M_2 \times \ldots M_D}$ is defined such that
\begin{equation*}
    \tensorize{\mat{w}}_{m_1 m_2 \cdots m_D} = w_m.
\end{equation*}

\subsection{Tensor Networks}

Tensor networks (TNs) \citep{kolda_tensor_2009,cichocki_era_2014,cichocki_tensor_2016,cichocki_tensor_2017} express a $D$-dimensional tensor ${\tensorize{\mat{w}}\eqqcolon\ten{W}}$ as a multi-linear function of $C$ \emph{core} tensors, see \cref{def:TensorNetwork} for a rigorous definition.
Two commonly used TNs are the canonical polyadic decomposition (CPD) and tensor train (TT).
\begin{definition}[Canonical polyadic decomposition \citep{hitchcock_expression_1927,kolda_tensor_2009}]\label{def:CPD}
A $D$-dimensional tensor $\ten{W}\in\mathbb{C}^{M_1\times M_2 \times \cdots \times M_D}$ has a rank-$R$ CPD if
\begin{equation*}\label{eq:CPD}
    w_{m_1m_2\dots m_D}= \sum_{r=0}^{R-1} \prod_{d=1}^D {w^{(d)}}_{m_d r}.
\end{equation*}
\end{definition}
The cores of this particular network are $C=D$ matrices {$\mat{W}^{(d)}\in\mathbb{C}^{M_d\times R}$}. The storage complexity $P=R\sum_{d=1}^D M_d$ of a rank-$R$ CPD is therefore \order{DMR}, where $M = \max (M_1, M_2,\ldots,M_D)$.

\begin{definition}[Tensor train \citep{oseledets_tensor-train_2011}]\label{def:TT}
A $D$-dimensional tensor $\ten{W}\in\mathbb{C}^{M_1 \times M_2 \times  \cdots \times  M_D}$ admits a rank-$(R_1\coloneqq1,R_2,\ldots,R_D,R_{D+1} \coloneqq R_1)$ tensor train if
\begin{equation*}
      w_{m_1m_2\ldots m_D} =  \sum_{r_1=0}^{R_1-1} \sum_{r_2=0}^{R_2-1} \cdots \sum_{r_D=0}^{R_d-1} \prod_{d=1}^D {w^{(d)}}_{r_{d} m_d r_{d+1}}.
\end{equation*}
\end{definition}
The cores of a tensor train are the $C=D$ $3$-dimensional tensors $\ten{W}^{(d)}\in\mathbb{C}^{R_{d}\times M \times R_{d+1}}$. The case $R_1 > 1$ is also called a tensor ring (TR)~\citep{zhao_tensor_2016}. 
Throughout the rest of this article we will simply refer to the tensor train rank as $R = \max (R_2, \cdots, R_{D})$. The storage complexity $P=\sum_{d=1}^D M_d R_d R_{d+1}$ of a tensor train is then \order{DMR^2}.
A TN is \emph{underparametrized} if $P\ll \prod_{d=1}^{D} M_d$, i.e. it can represent a tensor with fewer parameters than the number of entries of the tensor.

Other TNs are the Tucker decomposition~\citep{tucker_implications_1963-1,tucker_mathematical_1966}, hierarchical hierarchical Tucker \citep{hackbusch_new_2009,grasedyck_hierarchical_2010} decomposition, block-term decompositions \citep{de_lathauwer_decompositions_2008,de_lathauwer_decompositions_2008-1}, PEPS~\citep{verstraete_renormalization_2004} and MERA~\citep{evenbly_algorithms_2009}. 


\subsection{Tensorized Kernel Machines}
The tensor-product structure of features in~\cref{eq:TensorProductFeatureMap} can be exploited by imposing a tensor network structure onto the tensorized model weights
\begin{equation*}
    \tensorize{\mat{w}}.
\end{equation*}
Although generally speaking the tensorized  model weights are not full rank, modeling them as an underparametrized tensor network allows to compute fast model responses when the feature map $\mat{z}\left(\cdot\right)$ is of the form of \cref{eq:TensorProductFeatureMap}.
\begin{restatable}[Tensorized kernel machine (TKM)]{theorem}{nonQuantized}
\label{thm:CPD}
    Suppose $\tensorize{\mat{w}}$ is a tensor in CPD, TT or TR form. Then model responses and associated gradients
    \begin{equation*}
        f\left(\mat{x}\right) = \innerfrob{\bigotimes_{d=1}^D \mat{v}^{(d)}\left(x_d\right)}{\mat{w}},
    \end{equation*}
    can be computed in \order{P} instead of \order{\prod_{d=1}^D M_d}, where $P=DMR$ in case of CPD, and $P=DMR^2$ in case of TT or TR.
\end{restatable}
\begin{proof}
    See \cref{appendix:proofs_tkm}.
\end{proof}
Results for more general TNs can be found in \cref{appendix:proofs_tkm}.
This idea has been explored for a plethora of different combinations of tensor-product features and tensor networks \citep{wahls_learning_2014,stoudenmire_supervised_2016,novikov_exponential_2018,chen_parallelized_2018,cheng_supervised_2021,khavari_lower_2021,wesel_large-scale_2021}.
A graphical depiction of a TKM can be found in \cref{fig:TKM}: a full line denotes a summation along the corresponding index, while a dotted line denotes a Kronecker product.
Training a kernel machine under such constraint yields the following nonconvex optimization problem:
\begin{align}\label{eq:TensorKernelMachines}    & \min_{\mat{w}} \  \frac{1}{N}\sum_{n=1}^N \ell (\innerfrob{\bigotimes_{d=1}^D\mat{v}^{(d)}\left(x_d\right)}{\mat{w}},y_n) + \lambda \left\vert\left\vert \mat{w} \right\vert \right\vert^2, \\
\nonumber    & \text{s.t.} \  \tensorize{\mat{w}} \text{ is a tensor network}.
\end{align}
 Common choices of tensor network-specific optimizers are the alternating linear scheme (ALS)~\citep{comon_tensor_2009,kolda_tensor_2009,uschmajew_local_2012,holtz_alternating_2012}, the density matrix renormalization Group (DMRG)~\citep{white_density_1992} and Riemannian optimization~\citep{novikov_exponential_2018,novikov_automatic_2021}.
Generic first or second order gradient-based optimization method can also be employed.

\section{QUANTIZING POLYNOMIAL AND FOURIER FEATURES}
Before presenting the main contribution of this article, we first provide the definition of a pure-power polynomial feature map.

\begin{figure*}
    \hfill
    \centering
    \begin{subcaptionblock}[T][][c]{.33\textwidth}
    \centering
     \begin{tikzpicture}
    \SetVertexStyle[MinSize=1.1\DefaultUnit]
    \Vertex[x=0,y=0,shape=circle,label=$\mat{v}^{(1)}$,fontsize=\normalsize,RGB,color={0,114,189}]{phi1}
    \Vertex[x=1.5,y=0,shape=circle,label=$\mat{v}^{(2)}$,fontsize=\normalsize,RGB,color={217,83,25}]{phi2}
    \Vertex[x=3,y=0,shape=circle,label=$\mat{v}^{(3)}$,fontsize=\normalsize,RGB,color={237,177,32}]{phi3}
    \Vertex[x=0.0,y=-1.5,shape=rectangle,label=$\ten{W}^{(1)}$,fontsize=\normalsize,RGB,color={255,255,255}]{w1}
    \Vertex[x=1.5,y=-1.5,shape=rectangle,label=$\ten{W}^{(2)}$,fontsize=\normalsize,RGB,color={255,255,255}]{w2}
    \Vertex[x=3,y=-1.5,shape=rectangle,label=$\ten{W}^{(3)}$,fontsize=\normalsize,RGB,color={255,255,255}]{w3}
    \Edge[label=$M_1$,position=right](phi1)(w1)
    \Edge[label=$M_2$,position=right](phi2)(w2)
    \Edge[label=$M_3$,position=right](phi3)(w3)
    \Edge[label=$R_2$,position=below](w1)(w2)
    \Edge[label=$R_3$,position=below](w2)(w3)
    \Edge[style=dashed](phi1)(phi2)
    \Edge[style=dashed](phi2)(phi3)
\end{tikzpicture}
    \caption{TKM with TT-constrained weights.}
    \label{fig:TKM}
    \end{subcaptionblock}%
    \begin{subcaptionblock}[T][][c]{.66\textwidth}
    \centering
    \begin{tikzpicture}
    \SetVertexStyle[MinSize=1.1\DefaultUnit]
    \Vertex[x=00.0,y=0,shape=circle,label=$\mat{s}^{(1,1)}$,fontsize=\normalsize,RGB,color={0,114,189}]{phi11}
    \Vertex[x=01.5,y=0,shape=circle,label=$\mat{s}^{(1,2)}$,fontsize=\normalsize,RGB,color={0,114,189}]{phi12}
    \Vertex[x=03.0,y=0,shape=circle,label=$\mat{s}^{(1,3)}$,fontsize=\normalsize,RGB,color={0,114,189}]{phi13}
    \Vertex[x=04.5,y=0,shape=circle,label=$\mat{s}^{(2,1)}$,fontsize=\normalsize,RGB,color={217,83,25}]{phi21}
    \Vertex[x=06.0,y=0,shape=circle,label=$\mat{s}^{(3,1)}$,fontsize=\normalsize,RGB,color={237,177,32}]{phi31}
    \Vertex[x=07.5,y=0,shape=circle,label=$\mat{s}^{(3,2)}$,fontsize=\normalsize,RGB,color={237,177,32}]{phi32}

    \Vertex[x=00.0,y=-1.5,shape=rectangle,label=$\ten{W}^{(1,1)}$,fontsize=\normalsize,RGB,color={255,255,255}]{w11}
    \Vertex[x=01.5,y=-1.5,shape=rectangle,label=$\ten{W}^{(1,2)}$,fontsize=\normalsize,RGB,color={255,255,255}]{w12}
    \Vertex[x=03.0,y=-1.5,shape=rectangle,label=$\ten{W}^{(1,3)}$,fontsize=\normalsize,RGB,color={255,255,255}]{w13}
    \Vertex[x=04.5,y=-1.5,shape=rectangle,label=$\ten{W}^{(2,1)}$,fontsize=\normalsize,RGB,color={255,255,255}]{w21}
    \Vertex[x=06.0,y=-1.5,shape=rectangle,label=$\ten{W}^{(3,1)}$,fontsize=\normalsize,RGB,color={255,255,255}]{w31}
    \Vertex[x=07.5,y=-1.5,shape=rectangle,label=$\ten{W}^{(3,2)}$,fontsize=\normalsize,RGB,color={255,255,255}]{w32}
    
    \Edge[label=$Q$,position=right](phi11)(w11)
    \Edge[label=$Q$,position=right](phi12)(w12)
    \Edge[label=$Q$,position=right](phi13)(w13)
    \Edge[label=$Q$,position=right](phi21)(w21)
    \Edge[label=$Q$,position=right](phi31)(w31)
    \Edge[label=$Q$,position=right](phi32)(w32)
    \Edge[label=$R_2$,position=below](w11)(w12)
    \Edge[label=$R_3$,position=below](w12)(w13)
    \Edge[label=$R_4$,position=below](w13)(w21)
    \Edge[label=$R_5$,position=below](w21)(w31)
    \Edge[label=$R_6$,position=below](w31)(w32)
    \Edge[style=dashed](phi11)(phi12)
    \Edge[style=dashed](phi12)(phi13)
    \Edge[style=dashed](phi13)(phi21)
    \Edge[style=dashed](phi21)(phi31)
    \Edge[style=dashed](phi31)(phi32)
\end{tikzpicture}
    \caption{Corresponding QTKM with $Q$-quantized TT-constrained weights.}
    \label{fig:QTKM}
    \end{subcaptionblock}%
    \hfill
    \caption{TKM (\cref{fig:TKM}), and QTKM (\cref{fig:QTKM}) with TT-constrained model weights. In these diagrams, each circle represent a vector which constitutes the pure-power feature map of \cref{eq:PurePowerFeatureMap}, and each square represents a tensor train core (\cref{def:TT}). The color coding relates the $d$-th feature with its quantized representation. A full connecting line denotes a summation along the corresponding index, while a dotted line denotes a Kronecker product, see \citet{cichocki_tensor_2016} for a more in-depth explanation. \Cref{fig:QTKM} depicts the case where ${K_1=Q^3}$, $K_2 = Q$ and $K_3 = Q^2$.
   Notice how quantization allows to model correlations within each particular mode of the model weights, in this case explicitly by means of the tensor train ranks $(1,R_2,\ldots,R_6,1)$.}
    \label{fig:TKMs}
\end{figure*}

\begin{definition}[Pure-power polynomial feature map~{\citep{chen_parallelized_2018}}]\label{def:pureMfeature}
For an input sample $\mat{x} \in \mathbb{C}^D$, the pure-power polynomial features $\mat{z}(\cdot):\mathbb{C}^D \rightarrow \mathbb{C}^{M_1 M_2 \cdots M_D}$ of degree $\left(M_1-1,M_2-1,\ldots,M_D-1\right)$ are defined as
\begin{equation*}\label{eq:PurePowerFeatureMap}
    \mat{z}\left(\mat{x}\right) =  \bigotimes_{d=1}^D \mat{v}^{(d)}\left(x_d\right),
\end{equation*}
with
$\mat{v}^{(d)}\left(\cdot\right):\mathbb{C}\rightarrow\mathbb{C}^{M_d}$ the Vandermonde vector
\begin{equation*}\label{eq:VandermondeVector}
     \mat{v}^{(d)}\left(x_d\right) = \left[1,x_d,x_d^2,\ldots, x_d^{M_d-1}\right].
\end{equation*}
The $m_d$-th element of the feature map vector $\mat{v}^{(d)}(x_d)$ is
\begin{equation*}
    {v^{(d)}(x_d)}_{m_d} = (x_d)^{m_d}, \quad m_d =0,1,\ldots,M_d -1.
\end{equation*}
\end{definition}
 The definition of the feature map is given for degree $\left(M_1-1,M_2-1,\ldots,M_D-1\right)$ such that the feature map vector $z(\mat{x})$ has a length $M_1 M_2 \cdots M_D$. 
 The Kronecker product in \cref{def:pureMfeature} ensures that all possible combinations of products of monomial basis functions are computed, up to a total degree of $\sum_{d=1}^D (M_d-1)$.
Compared to the more common affine polynomials, which are basis functions of the polynomial kernel $k(\mat{x},\mat{x'})= (b+\inner{\mat{x}}{\mat{x'}})^M$, pure-power  polynomial features contain more higher-order terms.
Similarly, their use is justified by the Stone-Weierstrass theorem \citep{de_branges_stone-weierstrass_1959}, which guarantees that any continuous function on a locally compact domain can be approximated arbitrarily well by polynomials of increasing degree.
Fourier features can be similarly defined by replacing the monomials with complex exponentials.
\begin{definition}(Fourier Features)\label{def:FourierFeatureMap}
For an input sample {$\mat{x} \in \mathbb{C}^D$}, the Fourier feature map $\mat{\varphi}(\cdot):\mathbb{C}^D \rightarrow \mathbb{C}^{M_1 M_2 \cdots M_D}$ with $M_d$ basis frequencies {$-\nicefrac{M_d}{2},\ldots,\nicefrac{M_d}{2}-1$} per dimension is defined as
\begin{equation*}
    \mat{\varphi}\left(\mat{x}\right) =  \bigotimes_{d=1}^D \left( c_d  \,  \mat{v}^{(d)}\left(e^{-\frac{2\pi \, j \, x_d}{L}}\right) \right),
\end{equation*}
where $j$ is the imaginary unit, $c_d = e^{2\pi \, j\,x_d \, \frac{2+M_d}{2L}}\in\mathbb{C}$, $L\in\mathbb{R}$ is the periodicity of the function class and $\mat{v}^{(d)}\left(\cdot\right)$ are the Vandermonde vectors of \cref{eq:PurePowerFeatureMap}.
\end{definition}
Fourier features are ubiquitous in the field of kernel machines as they are eigenfunctions of $D$-dimensional stationary product kernels with respect to the Lebesgue measure, see \citep[Chapter 4.3]{rasmussen_gaussian_2006} or \citep{hensman_variational_2017,solin_hilbert_2020}. As such they are often used for the uniform approximation of such kernels in the limit of $L\rightarrow \infty$ and $M_1,M_2,\ldots,M_D\rightarrow\infty$ \citep[Proposition 1]{wahls_learning_2014}. 

We now present the first contribution of this article, which is an exact \emph{quantized}, i.e. further tensorized, representation of pure-power polynomials and Fourier features. These quantized features allows for the quantization of the model weights, which enables to impose additional tensor network structure between features, yielding more expressive models for the same number of model parameters.

\subsection{Quantized Features}

In order to quantize pure-power polynomial features we assume for ease of notation that $M_d$ can be written as some power ${M_d = Q^{K_d}}$, where both $Q,K_d\in\mathbb{N}$. The more general case involves considering the (prime) factorization of $M_d$ and follows the same derivation steps albeit with more intricate notation.
\begin{definition}[Quantized Vandermonde vector]\label{def:quantpure}
For $Q,k\in\mathbb{N}$, we define the quantized Vandermonde vector $\mat{s}^{(d,k)}(\cdot):\mathbb{C}\rightarrow \mathbb{C}^{Q}$ as
    \begin{equation*}
        \mat{s}^{(d,k)}\left(x_d\right)\coloneqq \left[1 , \, x_d^{{Q}^{k-1}},\ldots,x_d^{(Q-1)Q^{k-1}} \right].
    \end{equation*}
The $q$-th element of $\mat{s}^{(d,k)}(x_d)$ is therefore
\begin{equation*}
    {s^{(d,k)}(x_d)}_{q} = (x_d)^{q Q^{k-1}}, \quad q =0,1,\ldots,Q-1.
\end{equation*}
\end{definition}

\begin{theorem}[Quantized pure-power-$(M_d-1)$ polynomial feature map]\label{theo:quantpure}
Each Vandermonde vector $\mat{v}^{(d)}(x_d)$ can be expressed as a Kronecker product of $K_d$ factors
\begin{equation*}    
 \mat{v}^{(d)}(x_d)  = \bigotimes_{k=1}^{K_d} \mat{s}^{(d,k)}\left(x_d\right),
\end{equation*}
where $M_d = Q^{K_d}$.
\end{theorem}
\begin{proof}
From \cref{def:pureMfeature} we have that 
\begin{equation*}
{v^{(d)}(x_d)}_{m_d} = (x_d)^{m_d}.
\end{equation*}
Assume that $M_d = Q^{K_d}$. We proceed by tensorizing $\mat{v}^{(d)}(x_d)$ along $K_d$ dimensions, each having size $Q$. Then
\begin{align*}
    {v^{(d)}(x_d)}_{m_d} & = \tensorizeQuantized{v^{(d)}}_{q_1 q_2 \ldots q_{K_d}} \\
    &= (x_d)^{\sum_{k=1}^{K_d} q_{k} Q^{k-1}} \\
    &= \prod_{k=1}^{K_d} (x_d)^{q_{k}\,Q^{k-1}} \\
    &= \prod_{k=1}^{K_d} {s^{(d,k)}}(x_d)_{q_{k}}.
\end{align*}
The last equality follows directly from \cref{def:quantpure}. Hence by the definition of Kronecker product, we have that
\begin{equation*}    
 \mat{v}^{(d)}(x_d)  = \bigotimes_{k=1}^{K_d} \mat{s}^{(d,k)}\left(x_d\right).
\end{equation*}
\end{proof}
Note once more that in principle it is possible to tensorize with respect to $K_d$ indices such that $M_d = Q_1 Q_2 \cdots Q_{K_d}$, but we restrain from doing so not to needlessly complicate notation. \Cref{theo:quantpure} allows then to quantize pure-power and Fourier features.
\begin{corollary}[Quantized pure-power polynomials]\label{cor:poly}
    For an input sample $\mat{x}\in\mathbb{C}^{D}$, the pure-power polynomial feature map can be expressed as
    \begin{equation*}
        \mat{z}\left(\mat{x}\right) = \bigotimes_{d=1}^D \bigotimes_{k=1}^{K_d} \mat{s}^{(d,k)}\left(x_d\right).
    \end{equation*}
\end{corollary}
\begin{corollary}[Quantized Fourier feature map]\label{cor:FF}
For an input sample $\mat{x}\in\mathbb{C}^{D}$, the Fourier feature map can be expressed as\begin{equation*}
 \mat{\varphi}(\mat{x})  = \bigotimes_{d=1}^D \bigotimes_{k=1}^{K_d}  c_d^{\frac{1}{K_d}} \mat{s}^{(d,k)}\left(e^{-\frac{2\pi j x_d}{L}}\right),
\end{equation*}
where $c_d = e^{2\pi \, j\,x_d \, \frac{2+M_d}{2L}}$.
\end{corollary}
Note that when quantized, both pure-power and Fourier features admit an efficient storage complexity of \order{D K} = \order{D\log M} instead of \order{D M}, where $K= \max(K_1,\ldots,K_D)$.
\begin{example}
Consider $D=2$, $M_1=8=2^3$ $M_2=4=2^2$, then the Vandermonde vector of monomials up to total degree $10$ is constructed from
\begin{align*}
    \mat{z}(\mat{x}) &= \left[1,\,x_1 \right] \kron \left[1,\,x_1^2 \right] \kron \left[1,\ x_1^4 \right] \kron \left[1,\,x_2 \right] \kron \left[1,\,x_2^2 \right].
\end{align*}
\end{example}

We now present the second contribution of this article, which is the quantization of the model weights associated with quantized polynomial and Fourier features. As we will see, these quantized models are more expressive given the number of model parameters and same exact features.

\section{QUANTIZED TENSOR NETWORK KERNEL MACHINES}\label{section:quantizedModel}

When not considering quantization, model weights allow for tensorial indexing along the $D$ dimensions of the inputs, i.e. $\tensorize{\mat{w}}$. \Cref{cor:poly} and \cref{cor:FF} allow to exploit the Kronecker product structure of pure-power polynomial and Fourier features by further tensorizing the model weights of the tensor network-constrained kernel machines of \cref{eq:TensorKernelMachines}
\begin{equation*}
    \ff{ten}(\mat{w}, \underbrace{Q,Q,\ldots,Q}_{\sum_{d=1}^D K_d \ \text{times}}).
\end{equation*}
These further factorized model weights can then be constrained to be a tensor network, and learned by minimizing the empirical risk in the framework of \cref{eq:TensorKernelMachines}.
Training a kernel machine under this constraint results in the following nonlinear optimization problem:
\begin{align}\label{eq:TensorKernelMachinesQuantized}    & \min_{\mat{w}} \  \frac{1}{N}\sum_{n=1}^N \ell (\innerfrob{ \bigotimes_{d=1}^D \bigotimes_{k=1}^{K_d} \mat{s}^{(d,k)}\left(x_d\right)}{\mat{w}},y_n) + \lambda \left\vert\left\vert \mat{w} \right\vert \right\vert^2, \\
\nonumber    & \text{s.t.} \  \tensorizeQuantized{\mat{w}} \text{ is a tensor network}.
\end{align}

\subsection{Computational Complexity}

In case of CPD, TT or TR-constrained and quantized model weights, model responses and associated gradients can be computed at the same cost as with non-quantized models:
\begin{restatable}[Quantized tensorized kernel machine (QTKM)]{theorem}{quantized}\label{thm:CPDQuantized}
    Consider pure-power and Fourier feature maps factorized as in \cref{cor:poly} and \cref{cor:FF} and suppose $\tensorizeQuantized{\mat{w}}$ is a tensor in CPD, TT or TR form.
    Then by \cref{thm:CPD}, model responses and associated gradients
    \begin{equation*}
        f_{\text{quantized}}\left(\mat{x}\right) = \innerfrob{\bigotimes_{d=1}^D \bigotimes_{k=1}^{K_d} \mat{s}^{(d,k)}\left(x_d\right)}{\mat{w}},
    \end{equation*}
    can be computed in \order{P} instead of \order{\prod_{d=1}^D M_d}, where $P=KDQR$ in case of CPD, and $P= KDQR^2$ in case of TT or TR.
\end{restatable}
\begin{proof}
    See \cref{appendix:proofs_qtkm}.
\end{proof}
Results for more general TNs can be found in \cref{appendix:proofs_qtkm}.
A graphical depiction of a QTKM can be found in \cref{fig:QTKM}.
Furthermore, when considering tensor network-specific optimization algorithms, the time complexity per iteration of training when optimizing \cref{eq:TensorKernelMachinesQuantized} is lower compared to \cref{eq:TensorKernelMachines}, as these methods typically optimize over a subset (typically one core) of model parameters, see \cref{fasterALS}.

\subsection{Increased Model Expressiveness}

\begin{figure*}[h]
    \centering
    \begin{subcaptionblock}[T][][c]{.25\textwidth}
    \centering
    \includegraphics[width=\textwidth]{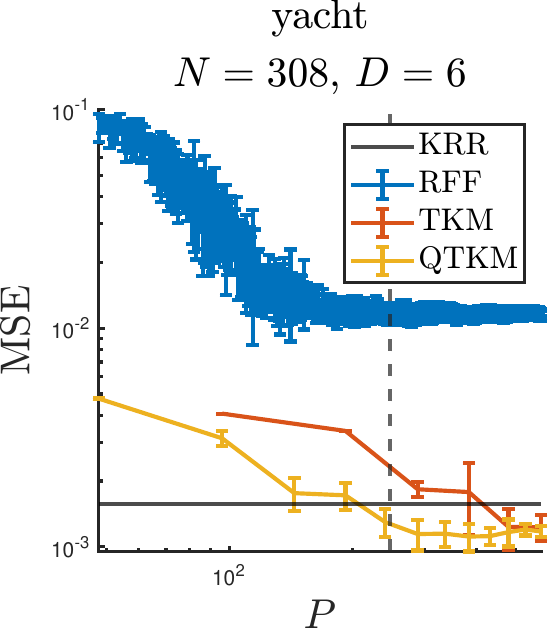}
    \end{subcaptionblock}%
    \begin{subcaptionblock}[T][][c]{.25\textwidth}
    \centering
    \includegraphics[width=\textwidth]{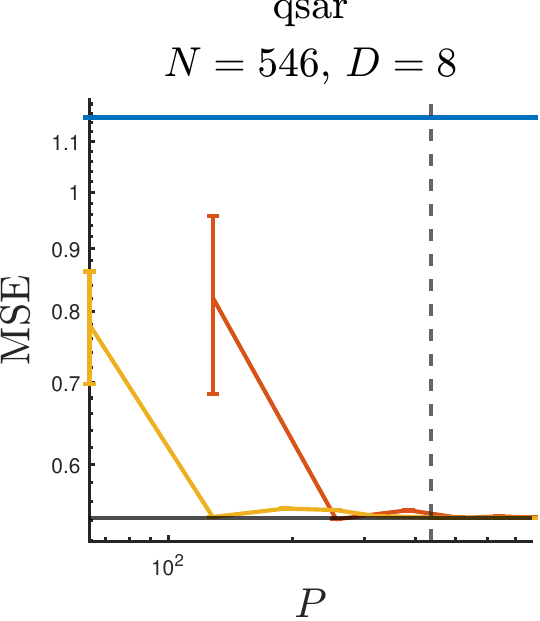}
    \end{subcaptionblock}%
    \begin{subcaptionblock}[T][][c]{.25\textwidth}
    \centering
    \includegraphics[width=\textwidth]{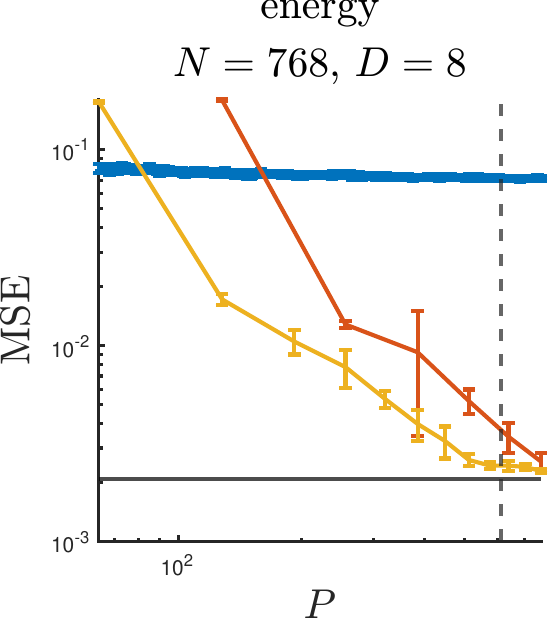}
    \end{subcaptionblock}%
    \begin{subcaptionblock}[T][][c]{.25\textwidth}
    \centering
    \includegraphics[width=\textwidth]{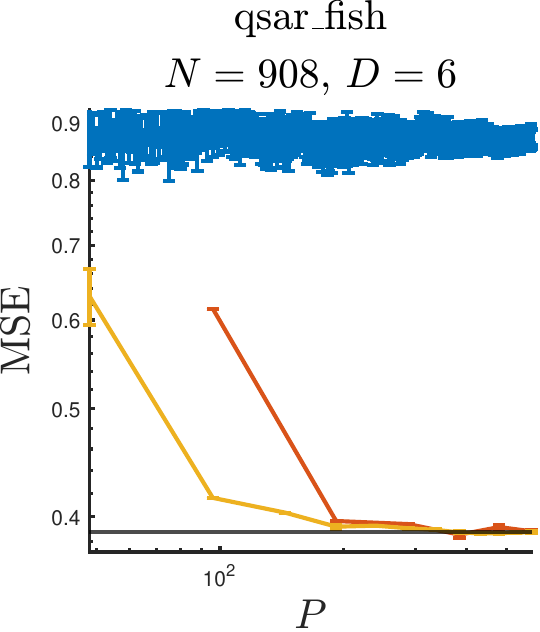}
    \end{subcaptionblock}\newline
    \begin{subcaptionblock}[T][][c]{.25\textwidth}
    \centering
    \includegraphics[width=\textwidth]{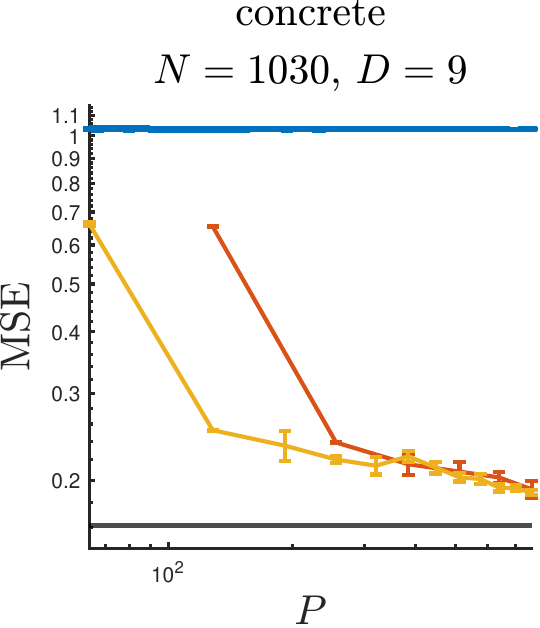}
    \end{subcaptionblock}%
    \begin{subcaptionblock}[T][][c]{.25\textwidth}
    \centering
    \includegraphics[width=\textwidth]{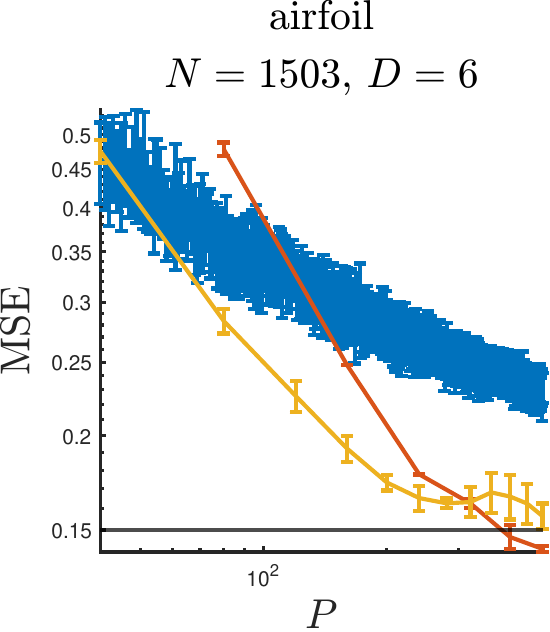}
    \end{subcaptionblock}%
    \begin{subcaptionblock}[T][][c]{.25\textwidth}
    \centering
    \includegraphics[width=\textwidth]{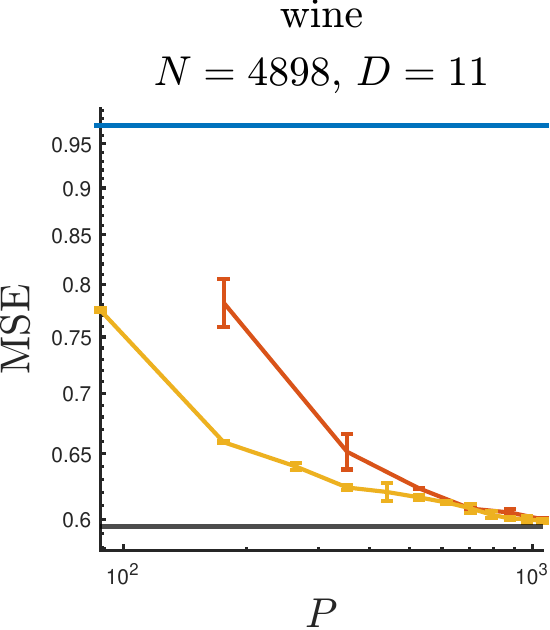}
    \end{subcaptionblock}%
    \begin{subcaptionblock}[T][][c]{.25\textwidth}
    \centering
    \includegraphics[width=\textwidth]{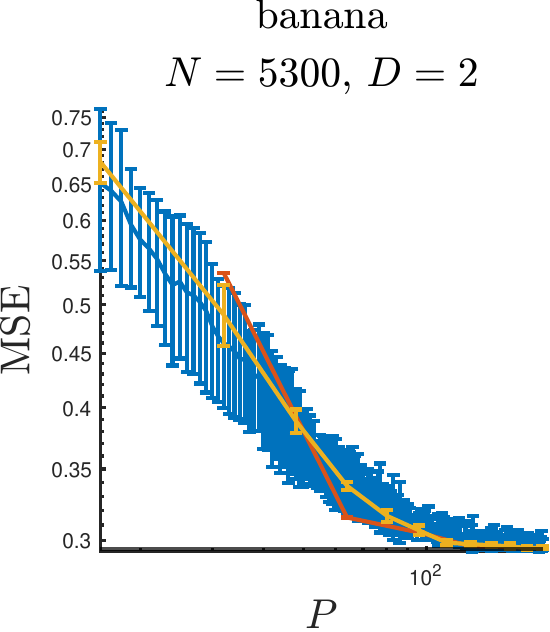}
    \end{subcaptionblock}%
    \caption{Plots of the test mean squared error as a function of the number of model parameters $P$, for different real-life datasets. In blue, random Fourier features \citep{rahimi_random_2007}, in red tensorized kernel machines with Fourier features \citep{wahls_learning_2014,stoudenmire_supervised_2016,kargas_supervised_2021,wesel_large-scale_2021}, in yellow quantized kernel machines with Fourier features, with quantization $Q=2$. The gray horizontal full line is the full unconstrained optimization problem, which corresponds to kernel ridge regression (KRR). The grey vertical dotted line is set at $P=N$. It can be seen that for $P<N$ case, quantization allows to achieve better generalization performance with respect to the non-quantized case.} 
    \label{fig:generalization}
\end{figure*}

Constraining a tensor to be a tensor network allows to distill the most salient characteristics of the data in terms of an limited number of effective parameters without destroying its multi-modal nature. This is also known as the \emph{blessing of dimensionality} \citep{cichocki_era_2014} and is the general underlying concept behind tensor network-based methods.
In the more specific context of supervised kernel machines, these well-known empirical considerations are also captured in the rigorous framework of VC-theory \citep{vapnik_nature_1998}. \citet[theorem 2]{khavari_lower_2021} have recently shown that the VC-dimension and pseudo-dimension of tensor network-constrained models of the form of \cref{eq:TensorKernelMachinesQuantized} satisfies the following upper bound \emph{irrespectively of the choice of tensor network}:
\begin{equation*}\label{eq:VC_general}
    \text{VC}(f) \leq 2P \log(12 \vert V\vert),
\end{equation*}
where $\vert V \vert$ is the number of vertices in the TN (see \cref{def:TensorNetwork}).
Since quantization of the model weights increases the number of vertices in their tensor network representation, quantized models are characterized by higher upper bounds on the VC-dimension and pseudo-dimension \emph{for the same number of model parameters}.
For example, in the non-quantized case, parametrizing the TN as a CPD, TT or TR yields
\begin{equation*}\label{eq:VC}
    \text{VC}(f) \leq 2P \log(12 D),
\end{equation*}
while for the quantized case
\begin{equation*}\label{eq:VCQuantized}
    \text{VC}(f_\text{quantized}) \leq 2P\log(12 D \log M).
\end{equation*}
Hence, in case of CPD, TT and TR this additional possible model expressiveness comes at \emph{no additional computational costs per iteration} when training with gradient descent (\cref{thm:CPD,thm:CPDQuantized}). Setting $Q=2$ provides then in this sense an optimal choice for this additional hyperparameter, as it maximizes the upper bound. In the more general case where $M_d$ is not a power of $2$, this choice corresponds with the prime factorization of $M_d$.
It should be noted that a higher VC-dimension does not imply better performance on unseen data. However as we will see in \cref{sec:generalization,sec:regularization} quantized models tend to outperform their counterparts in the underparameterized regime where TKMs are typically employed, as the gained expressiveness is put fully to good use and does not result in overfitting.



\begin{figure*}[h]
    \centering
    \begin{subcaptionblock}[T][][c]{.25\textwidth}
    \centering
    \includegraphics[width=\textwidth]{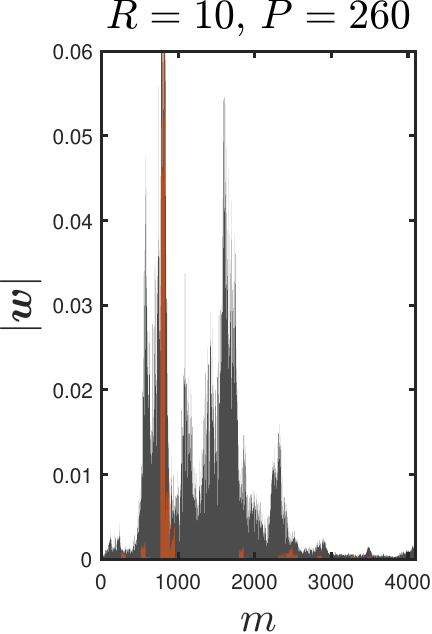}
    \end{subcaptionblock}%
    \begin{subcaptionblock}[T][][c]{.25\textwidth}
    \centering
    \includegraphics[width=\textwidth]{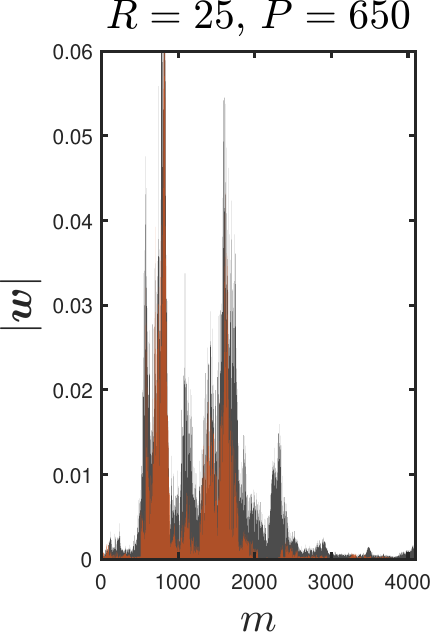}
    \end{subcaptionblock}%
    \begin{subcaptionblock}[T][][c]{.25\textwidth}
    \centering
    \includegraphics[width=\textwidth]{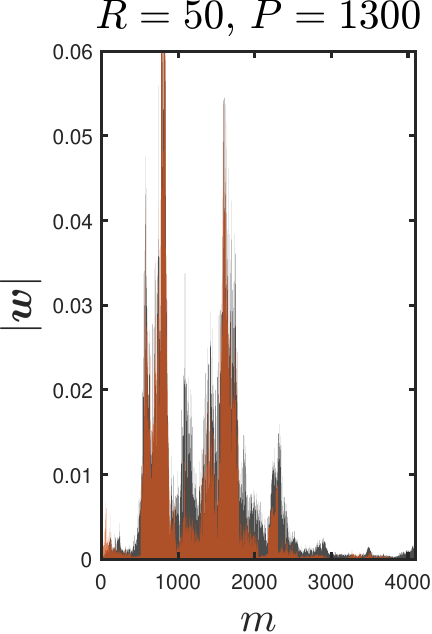}
    \end{subcaptionblock}%
    \begin{subcaptionblock}[T][][c]{.25\textwidth}
    \centering
    \includegraphics[width=\textwidth]{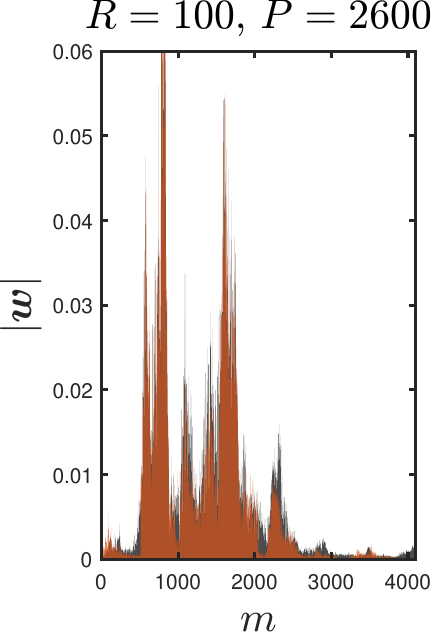}
    \end{subcaptionblock}%
    \caption{Sound dataset. In red, plot of the magnitude of the quantized Fourier coefficients for different values of $R$ and total number of model parameters $P$. The magnitude of the full unconstrained Fourier coefficients is shown in black. It can be observed that increasing the CPD rank $R$ recovers the peaks of frequencies with the highest magnitude.}
    \label{fig:sound}
\end{figure*}

\section{NUMERICAL EXPERIMENTS}

In all experiments we consider a squared loss $\ell (f(\mat{x}),y)=\vert f(\mat{x})-y\vert^2$, scale our inputs to lie in the unit box, and consider Fourier features (\cref{def:FourierFeatureMap}) as they notably suffer less from ill-conditioning than polynomials. In all experiments we model the weight tensor as a CPD of rank $R$. We do not consider other TNs in the numerical experiments for three reasons: 
first, it has been shown that tensor trains are more suited to model time-varying functions such as dynamical systems and time series, as opposed to CPD \citep{khrulkov_expressive_2018}.
Second, CPD adds only one hyperparameter to our model as opposed to $D$ hyperparameters for the tensor train or tensor ring. Choosing these
hyperparameters (tensor train ranks) is not trivial and can yield models with very different performance for the same total number of model parameters.
Third, CPD-based models are invariant to reordering of the features as opposed to tensor train. We believe that this invariance is very much desired in the context of kernel machines.
We solve the ensuing optimization problem using ALS \citep{uschmajew_local_2012}. The source code and data to reproduce all experiments is available at \url{https://github.com/fwesel/QFF}.

\subsection{Improved Generalization Capabilities}\label{sec:generalization}
\begin{table*}[h]
\centering
\sisetup{table-align-uncertainty=true,
        retain-zero-uncertainty=true,
        separate-uncertainty=true,
        scientific-notation = false,
        table-alignment-mode=none}
\caption{MSE for different kernel machines on the airline dataset with one standard deviation. We report the number of basis functions $M$ per dimensions (in case of random approaches we simply report the total number of basis) and model parameters $P$. 
Notice that QTKM is able to parsimoniously predict airline delay with a restricted number of model parameters, achieving state-of-the art performance on this dataset.}
\robustify \bfseries
\label{tbl:large}
\begin{tabular}{
l
S[table-number-alignment = center]
S[table-number-alignment = center]
S[table-format=1.3(4),detect-weight]
}
\toprule                           
{Method}    &   {$M$}   &   {$P\downarrow$}   &   {MSE} \\
\cmidrule{1-4}
VFF \citep{hensman_variational_2017}              &   40    &  320   &   0.827 \pm 0.004 \\
Hilbert-GP \citep{solin_hilbert_2020}             &   40    &  320   &   0.827 \pm 0.005 \\
VISH \citep{dutordoir_sparse_2020}                &   660    &  660   &   0.834 \pm 0.055 \\
SVIGP \citep{hensman_gaussian_2013}               &   1000   &  1000  &   0.791 \pm 0.005 \\
Falkon \citep{meanti_kernel_2020}                 &   10000  &  10000 &   0.758 \pm 0.005 \\
\cmidrule{1-4}
TKM ($R=4$)     &   64     &  2048  &   0.789 \pm 0.005 \\    
TKM ($R=6$      &   64     &  3072  &   0.773 \pm 0.006 \\    
TKM ($R=8$)     &   64     &  4096  &   0.765 \pm 0.007 \\    
\cmidrule{1-4}
QTKM ($R=20$)        &   64     &  1920  &   0.764 \pm 0.005 \\
QTKM ($R=30$)        &   64     &  2880  &   0.754 \pm 0.005 \\
QTKM ($R=40$)        &   64     &  3840  &  \bfseries 0.748 \pm 0.005 \\
\bottomrule
\end{tabular}
\end{table*}

In this experiment we verify the expected quantization to positively affect the generalization capabilities of quantized models. We compare QTKM (our approach) with TKM \citep{wahls_learning_2014,stoudenmire_supervised_2016,kargas_supervised_2021,wesel_large-scale_2021}, random Fourier features (RFF) \citep{rahimi_random_2007}, and with the full, unconstrained model (kernel ridge regression (KRR) which is our baseline, as we are dealing in all cases with squared loss). 
For our comparison we select eight small UCI datasets \citep{dua_uci_2017}. This choice allows us to train KRR by solving its dual optimization problem and thus to implicitly consider $\prod_{d=1}^D M_d$ features. For each dataset, we select uniformly at random $\qty{80}{\percent}$ of the data for training, and keep the rest for test.
We set $Q=2$ and select the remaining hyperparameters ($\lambda$ and $L$) by $\num{3}$-fold cross validating KRR. We set the number of basis functions $M_d=16$ uniformly for all $d$ for all models, so that they learn from the same representation (except for RFF, which is intrinsically random). We then vary the rank $R$ of the non-quantized tensorized model from $R=1,2,\ldots,6$ and train all other models such that their number of model parameters $P$ is at most equal to the ones of the non-quantized model.
This means that for TKM $P=R\sum_{d=1}^D M_d$, for QTKM $P=2R\sum_{d=1}^D \log_2 M_d$ and for RFF $P$ equals the number of random frequencies. To make sure that TKM and QTKM converge, we run ALS for a very large number of iterations ($\num{5000}$).
We repeat the procedure $\num{10}$ times, and plot the mean and standard deviation of the test mean squared error (MSE) in \cref{fig:generalization}.

In \cref{fig:generalization} one can observe that on all datasets, for the same number of model parameters $P$ and identical features, the generalization performance of QTKM is equivalent or better in term of test MSE. An intuitive explanation for these results is that for equal $P$, quantization allows to explicitly model correlations within each of the $D$ modes of the feature map, yielding models with increased learning capacity.
We notice that while on most datasets the tensor-based approaches recover the performance of KRR, in one case, namely on the yacht dataset, the performance is better than baseline, pointing out at the regularizing effect of the quantized CPD model.
Furthermore, on all datasets examined in \cref{fig:generalization} it can be observed that QTKM switches from underfitting to overfitting regime (first local optimum of the learning curve) before TKM, indicating that indeed its capacity is saturated with less model parameters. At that sweet spot, TKM is still underfitting and underperforming with respect to QTKM. For a further increase in model parameters both models exhibit double descent, as can be observed on the qsar, qsar\_fish and airfoil datasets. Note that QTKM outperforms TKM in a similar fashion on the training set (\cref{fig:generalization_train} in the appendix), corroborating the presented analysis.
In \cref{fig:generalization} it can also be seen that except on the examined $\num{2}$-dimensional dataset, both tensor network are consistently outperforming RFF. As we will see in \cref{sec:regularization}, these tensor network-based methods are able to find in a data-dependent way a parsimonious model representation given an exponentially large feature space. This is in contrast to random methods such as RFF, which perform feature selection prior to training and are in this sense oblivious to training data.

\subsection{Regularizing Effect of Quantization}\label{sec:regularization}

We would like to gain insight in the regularizing effect caused by modeling the quantized weights as an underparametrized tensor network. For this reason we investigate how the Fourier coefficients are approximated as a function of the CPD rank in a one-dimensional dataset. In order to remove other sources of regularization, we set $\lambda=0$. The sound dataset \citep{wilson_kernel_2015} is a one-dimensional time series regression task which comprises $\num{60000}$ sampled points of a sound wave. The training set consists of $N=\num{59309}$ points, of which the remainder is kept for test.
Based on the Nyquist–Shannon sampling theorem, we consider $M=2^{13}=8192$ Fourier features, which we quantize with $Q=2$. We model the signal as a having unit period, hence set $L=1$. The Fourier coefficients are modeled as a CPD tensor, with rank $R=10,25,50,100$ in order to yield underparametrized models ($P\ll M$).
We plot the magnitude of the Fourier coefficients, which we obtain by minimizing \cref{eq:TensorKernelMachinesQuantized} under squared loss.

We compare the magnitude of the quantized weights with the magnitude of the unconstrained model response, obtained by solving \cref{eq:EmpiricalRisk}, in \cref{fig:sound}. From \cref{fig:sound} we can see that for low values of $R$ the quantized kernel machine does not recover the coefficients associated with the lowest frequencies, as a data-independent approach would. Instead, we observe that the coefficients which are recovered for lower ranks, e.g. in case of $R=10$, are the peaks with the highest magnitude. This is explained by the fact that the additional modes introduced by $Q=2$-quantization force the underparametrized tensor network to model the nonlinear relation between different basis which under squared-loss maximize the energy of the signal. As the rank increases, the increased model flexibility allows to model more independent nonlinearities. We can see that already for $R=100$ the two spectra become almost indistinguishable. We report the relative approximation error of the weights and the standardized mean absolute error on the test set in \cref{appendix:sound}.



\subsection{Large-Scale Regression}\label{sec:large-scale}

In order to showcase and compare out approach with existing literature in the realm of kernel machines, we consider the airline dataset \citep{hensman_gaussian_2013}, an $8$-dimensional dataset which consists of $N=\num{5929413}$ recordings of commercial airplane flight delays that occurred in \num{2008} in the USA. 
As is standard on this dataset \citep{samo_string_2016}, we consider a uniform random draw of $\nicefrac{2}{3} N$ for training and keep the remainder for the evaluation of the MSE on the test set and repeat the procedure ten times.
In order to capture the complicated nonlinear relation between input and output, we resort to consider $M_d=64$ Fourier features per dimension, which we quantize with $Q=2$. For this experiment, we set $L=10$, $\lambda=\num{1e-10}$ and run the ALS optimizer for $\num{25}$ epochs. We train three different QTKMs with $R=20,30,40$. 

We present the results in \cref{tbl:large}, where we can see that QTKM (our approach) is best at predicting airline delay in term of MSE. Other grid-based approaches, such as VFE \citep{hensman_variational_2017} or Hilbert-GP \citep{solin_hilbert_2020}, are forced to resort to additive kernel modeling and thus disregard higher-order interactions between Fourier features pertaining to different dimension. In contrast, QTKM is able to construct $R$ data-driven explanatory variables based on an exponentially large set of Fourier features. When compared with its non-quantized counterpart TKM, we can see that our quantized approach outperforms it with approximately half of its model parameters.
Training QTKM on the Intel Core i7-10610U CPU of a Dell Inc. Latitude 7410 laptop with $\qty{16}{\giga\byte}$ of RAM took $\qty{6613 \pm 40}{\second}$ for $R=20$ and took $\qty{13039 \pm 114}{\second}$ for $R=40$ .

\section{CONCLUSION}
We proposed to quantize Fourier and pure-power polynomial features, which allowed us to quantize the model weights in the context of tensor network-constrained kernel machines. We verified experimentally the theoretically expected increase in model flexibility which allows us to construct more expressive models with the same number of model parameters which learn from the same exact features at the same computational cost per iteration.

Our approach can be readily incorporated in other tensor network-based learning methods which make use of pure-power polynomials or Fourier features.

\subsubsection*{Acknowledgments}
We would like to thank the anonymous reviewers and Albert Saiapin for their numerous suggestions and improvements which have greatly improved the quality of this paper.
Frederiek Wesel, and thereby this work, is supported by the Delft University of Technology AI Labs
program. The authors declare no competing interests.

\bibliography{Bibliography.bib}

\section*{Checklist}

 \begin{enumerate}

 \item For all models and algorithms presented, check if you include:
 \begin{enumerate}
   \item A clear description of the mathematical setting, assumptions, algorithm, and/or model. Yes.
   \item An analysis of the properties and complexity (time, space, sample size) of any algorithm. Yes.
   \item (Optional) Anonymized source code, with specification of all dependencies, including external libraries. Yes.
 \end{enumerate}

 \item For any theoretical claim, check if you include:
 \begin{enumerate}
   \item Statements of the full set of assumptions of all theoretical results. Yes.
   \item Complete proofs of all theoretical results. Yes.
   \item Clear explanations of any assumptions. Yes.   
 \end{enumerate}

 \item For all figures and tables that present empirical results, check if you include:
 \begin{enumerate}
   \item The code, data, and instructions needed to reproduce the main experimental results (either in the supplemental material or as a URL). Yes.
   \item All the training details (e.g., data splits, hyperparameters, how they were chosen). Yes.
         \item A clear definition of the specific measure or statistics and error bars (e.g., with respect to the random seed after running experiments multiple times). Yes.
         \item A description of the computing infrastructure used. (e.g., type of GPUs, internal cluster, or cloud provider). Yes.
 \end{enumerate}

 \item If you are using existing assets (e.g., code, data, models) or curating/releasing new assets, check if you include:
 \begin{enumerate}
   \item Citations of the creator If your work uses existing assets. Yes.
   \item The license information of the assets, if applicable. Yes.
   \item New assets either in the supplemental material or as a URL, if applicable. Yes. All necessary assets can be found in the anonymized URL.
   \item Information about consent from data providers/curators. Not Applicable (UCI data).
   \item Discussion of sensible content if applicable, e.g., personally identifiable information or offensive content. Not applicable. All data is anonymous and open source.
 \end{enumerate}

 \item If you used crowdsourcing or conducted research with human subjects, check if you include:
 \begin{enumerate}
   \item The full text of instructions given to participants and screenshots. Not Applicable.
   \item Descriptions of potential participant risks, with links to Institutional Review Board (IRB) approvals if applicable. No/Not Applicable.
   \item The estimated hourly wage paid to participants and the total amount spent on participant compensation. Not Applicable.
 \end{enumerate}

 \end{enumerate}

\newpage
\appendix
\onecolumn

\section{DEFINITIONS}

\begin{definition}[Tensor network \citep{khavari_lower_2021}]\label{def:TensorNetwork}
    Given a graph $G = (V,E, \text{dim})$ where $V$ is a set of vertices, $E$ is a set of edges and $\text{dim}:E\rightarrow \mathbb{N}$ assigns a dimension to each edge, a tensor network assigns a core tensor $\ten{C}_v$ to each vertex of the graph, such that $\ten{C}_v\in \kron_{e\in E_v}\mathbb{C}^{\text{dim}(e)}$. Here $E_v=\{e\in E \vert v\in e\}$ is the set of edges connected to vertex $v$.
    The resulting tensor is a tensor in $\kron_{e\in E \cap  V}\mathbb{C}^{\text{dim}(e)}$.
    The number of parameters of the tensor network is then $P=\sum_{v\in V}\prod_{e\in E_v} \text{dim}(e)$.
\end{definition}

\section{PROOFS}
\subsection{Tensor Network Kernel Machine}\label{appendix:proofs_tkm}

\begin{theorem}
    Suppose $\tensorizeTimes{\mat{w}}$ is a tensor network. Then the dependency on $M$ of the computational complexity for the model responses
    \begin{equation*}
        f\left(\mat{x}\right) = \innerfrob{\bigotimes_{d=1}^D \mat{v}^{(d)}\left(x_d\right)}{\mat{w}},
    \end{equation*}
    is \order{M^t}, where $t$ is the maximum number of singleton edges per core.
\end{theorem}
\begin{proof}
    Let $t$ be the maximum number of singleton edges per core. Since taking the Frobenius inner product (\cref{eq:inner}) involves summing over all singleton edges $\underbrace{M,M,\ldots,M}_{D \ \text{times}}$, the required number of FLOPS will be \order{M^t}.
\end{proof}

\begin{corollary}\label{thm:TKMComplexity2}
Suppose $\tensorizeTimes{\mat{w}}$ is a tensor network with $t=1$ maximum number of singleton edges per core. Then the dependency on $M$ of the computational complexity for the model responses
    \begin{equation*}
        f\left(\mat{x}\right) = \innerfrob{\bigotimes_{d=1}^D \mat{v}^{(d)}\left(x_d\right)}{\mat{w}},
    \end{equation*}
    is of \order{M}.
\end{corollary}

Note that most used tensor networks such as CPD, Tucker, TT/TR, MERA, PEPS have $t=1$. An example of a tensor network where $t$ can be $t\geq2$ or higher is hierarchical Tucker. In what follows we derive the computational complexity of the model responses of CPD and TT networks.

\nonQuantized*
\begin{proof}\label{proof:1}
    Let $\tensorize{\mat{w}}$ be a tensor in CPD form. Then
    \begin{align*}
    f\left(\mat{x}\right) &= \innerfrob{\bigotimes_{d=1}^D \mat{v}^{(d)}\left(x_d\right)}{\mat{w}} \\
    &= \sum_{m_1=0}^{M_1-1}\cdots \sum_{m_D=0}^{M_D-1} \prod_{d=1}^D v^{(d)}_{m_d} \sum_{r=0}^{R-1} \prod_{d=1}^D w^{(d)}_{m_d r} \\
    &= \sum_{r=0}^{R-1} \sum_{m_1=0}^{M_1-1} \cdots \sum_{m_D=0}^{M_D-1} \prod_{d=1}^D v_{m_d}^{(d)} w_{m_d r}^{(d)} \\
    &= \sum_{r=0}^{R-1} \prod_{d=1}^D \sum_{m_d=0}^{M_d-1} v_{m_d}^{(d)} w_{m_d r}^{(d)}. \\
\end{align*}
Gradients can be computed efficiently by caching $\prod_{d=1}^D\sum_{m_d=0}^{M_d-1}  v_{m_d}^{(d)} w_{r_{d-1} m_d r_{d}}^{(d)}, \ r=1,\ldots,R$.
Hence the computational complexity of the model responses and associated gradients is of \order{DMR}. 
Now let $\tensorize{\mat{w}}$ be a tensor in TT/TR form. Then
    \begin{align*}
    f\left(\mat{x}\right) &= \innerfrob{\bigotimes_{d=1}^D \mat{v}^{(d)}\left(x_d\right)}{\mat{w}} \\
    &= \sum_{m_1=0}^{M_1-1}\cdots \sum_{m_D=0}^{M_D-1} \prod_{d=1}^D v^{(d)}_{m_d} \sum_{r_1=0}^{R_1-1}\cdots\sum_{r_D=0}^{R_D-1} \prod_{d=1}^D w^{(d)}_{r_{d-1} m_d r_{d}} \\
    &= \sum_{r_1=0}^{R_1-1}\cdots\sum_{r_D=0}^{R_D-1} \sum_{m_1=0}^{M_1-1} \cdots \sum_{m_d=0}^{M_D-1} \prod_{d=1}^D v_{m_d}^{(d)} w^{(d)}_{r_{d-1} m_d r_{d}} \\
    &= \sum_{r_1=0}^{R_1-1}\cdots\sum_{r_D=0}^{R_D-1} \prod_{d=1}^D \sum_{m_d=0}^{M_d-1} v_{m_d}^{(d)} w_{r_{d-1} m_d r_{d}}^{(d)}, \\
\end{align*}
which is a sequence of matrix-matrix multiplications. Gradients can be computed efficiently by caching $\sum_{m_d=0}^{M_d-1}  v_{m_d}^{(d)} w_{r_{d-1} m_d r_{d}}^{(d)}, \ r_{d}=1,\ldots,R_{d}, \ d=1,\ldots,D$.
Hence the computational complexity of the model responses and associated gradients is of \order{DMR^2}, where $M=\max(M_1,M_2,\ldots,M_D)$ i.e. \order{P} for both CPD and TT/TR.
\end{proof}

\subsection{Quantized Tensor Network Kernel Machine}\label{appendix:proofs_qtkm}

\begin{theorem}\label{thm:QTKMComplexity1}
    Suppose $\tensorizeQuantizedTimes{\mat{w}}$ is a tensor network. Then the dependency on $M$ on the computational complexity of model responses
    \begin{equation*}
        f\left(\mat{x}\right) = \innerfrob{\bigotimes_{d=1}^D \bigotimes_{k=1}^{K} \mat{s}^{(d,k)}\left(x_d\right)}{\mat{w}},    \end{equation*}
    is of \order{M^{\frac{t}{\log_{Q}}} \log M}, where $t$ is the maximum number of singleton edges per core.
\end{theorem}
\begin{proof}
    Let $Q$ be chosen such that $K = \log_Q M$. Let $t$ be the maximum number of singleton edges per core. Taking the Frobenius inner product (\cref{eq:inner}) involves summing over all singleton edges $\underbrace{Q,Q,\ldots,Q}_{D \log_Q M \ \text{times}}$. Since $Q=\frac{1}{\log_Q M}$, the required number of FLOPS will be \order{Q^t \log_Q M} = \order{M^{\frac{s}{\log_{Q}}}\log M}.
\end{proof}

\begin{corollary}\label{thm:QTKMComplexity2}
Suppose $\tensorizeQuantizedTimes{\mat{w}}$ is a tensor network with $t=1$ maximum number of singleton edges per core. Then the dependency on $M$ on the computational complexity of model responses
    \begin{equation*}
        f\left(\mat{x}\right) = \innerfrob{\bigotimes_{d=1}^D \bigotimes_{k=1}^{K_d} \mat{s}^{(d,k)}\left(x_d\right)}{\mat{w}},    \end{equation*}
    is of \order{\log M}.
\end{corollary}

\quantized*
\begin{proof}
     The proof follows from the proof of \cref{thm:CPD}. Since instead of summing $R$ times over $M_1,M_2,\ldots,M_D$ we are summing $R$ times over $\underbrace{Q,Q,\ldots,Q}_{D K \ \text{times}}$, a model response can be evaluated in $QKDR$ FLOPS for CPD and $QKDR^2$ FLOPS for TT. Since $Q$ is a constant which does not dependent on $M$ and $K=\log_Q M$, we have that the computational complexities are respectively \order{\log M DR} and \order{\log M DR^2} for CPD and TT/TR, where $K=\log M=\max(\log M_1, \log M_2,\ldots,\log M_D)$, i.e. \order{P} for both CPD and TT/TR.
\end{proof}

\section{FASTER MULTI-CONVEX OPTIMIZATION ALGORITHMS}\label{fasterALS}
Quantized features allow to speedup \cref{eq:TensorKernelMachinesQuantized} for a large class of multi-convex solvers such as alternating least-squares~\citep{comon_tensor_2009,kolda_tensor_2009,uschmajew_local_2012,holtz_alternating_2012}, the density matrix renormalization group (DMRG)~\citep{white_density_1992} and Riemannian optimization~\citep{novikov_exponential_2018,novikov_automatic_2021}. These solvers exploit the multi-linearity of tensor networks in order to express the empirical risk as a function of only one core of the weight tensor in tensor network form per iteration, also known as sub-problem. After solving the ensuing optimization sub-problem, this procedure is repeated for for the remaining cores, defining one epoch. The whole procedure is then repeated until convergence. 
When a convex quadratic loss function is used, computational benefits associated with quantization arise as it enables to solve a series of quadratic problems exactly. This is common practice in literature, see for instance \citet{wahls_learning_2014,chen_parallelized_2018,novikov_exponential_2018,wesel_large-scale_2021}.

In the exemplifying case of CPD, TT and tensor ring, for a fixed number of model parameters $P$, quantization allows to solve each sub-problem at a reduced computational cost of \order{\nicefrac{P^2}{D^2}} compared to a cost of \order{\nicefrac{P^2}{D^2 (\log M)^2}}. This yields a sub-problem complexity which is \emph{independent} of $M$. 
A similar reduction follows for other one-layered networks. Quantifying the computational gains for other structures of tensor networks is less straightforward.

\section{NUMERICAL EXPERIMENTS}

\subsection{Improved Generalization Capabilities}\label{appendix:generalization}
We report in \cref{fig:generalization_train} the training error on the examined datasets in \cref{sec:generalization}. As one can observe, QTKM outperforms TKM in terms of training error (\cref{fig:generalization_train}) and test error (\cref{fig:generalization}).

\begin{figure*}[h]
    \centering
    \begin{subcaptionblock}[T][][c]{.25\textwidth}
    \centering
    \includegraphics[width=\textwidth]{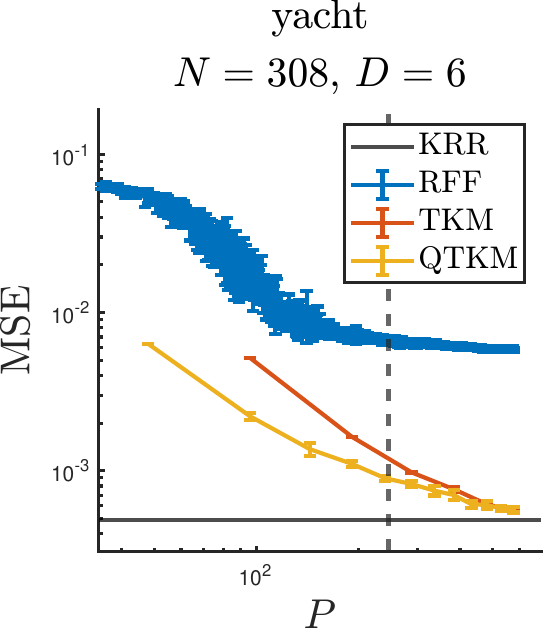}
    \end{subcaptionblock}%
    \begin{subcaptionblock}[T][][c]{.25\textwidth}
    \centering
    \includegraphics[width=\textwidth]{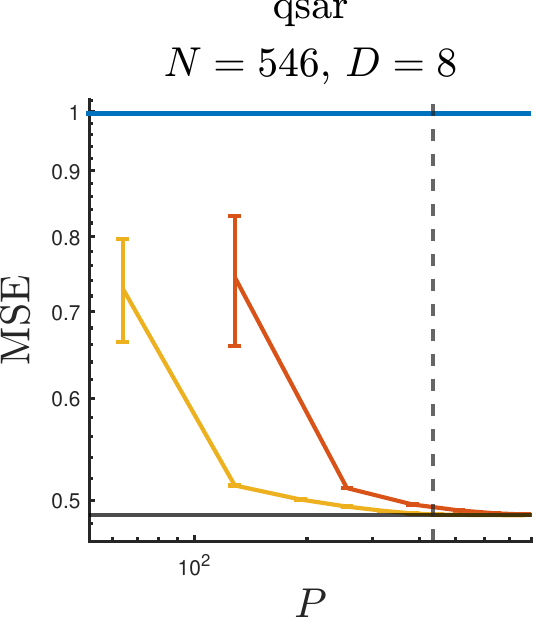}
    \end{subcaptionblock}%
    \begin{subcaptionblock}[T][][c]{.25\textwidth}
    \centering
    \includegraphics[width=\textwidth]{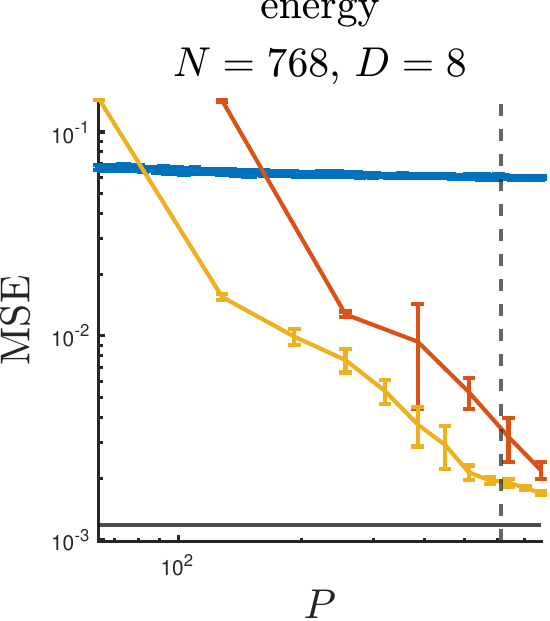}
    \end{subcaptionblock}%
    \begin{subcaptionblock}[T][][c]{.25\textwidth}
    \centering
    \includegraphics[width=\textwidth]{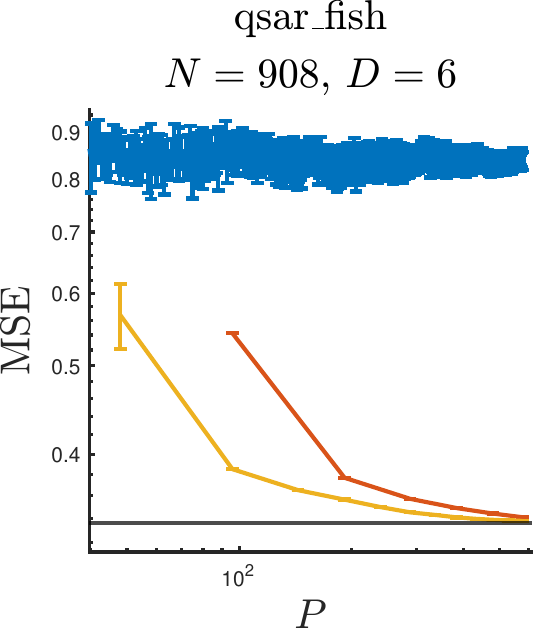}
    \end{subcaptionblock}\newline
    \begin{subcaptionblock}[T][][c]{.25\textwidth}
    \centering
    \includegraphics[width=\textwidth]{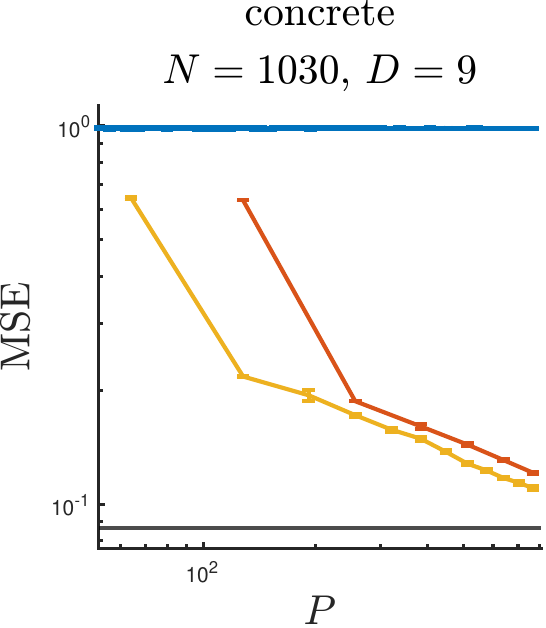}
    \end{subcaptionblock}%
    \begin{subcaptionblock}[T][][c]{.25\textwidth}
    \centering
    \includegraphics[width=\textwidth]{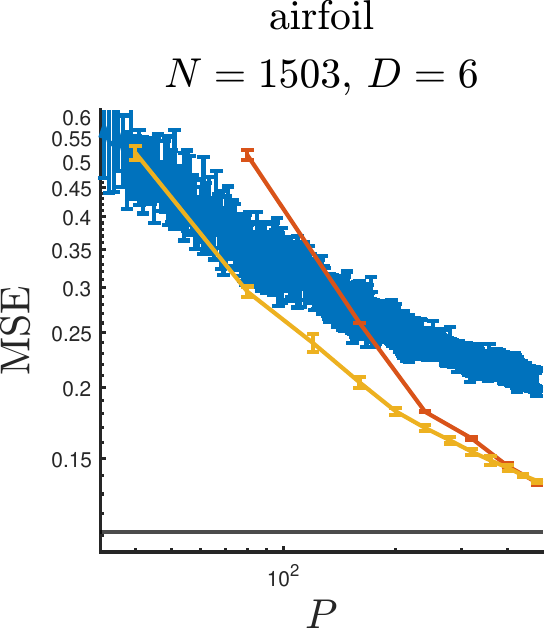}
    \end{subcaptionblock}%
    \begin{subcaptionblock}[T][][c]{.25\textwidth}
    \centering
    \includegraphics[width=\textwidth]{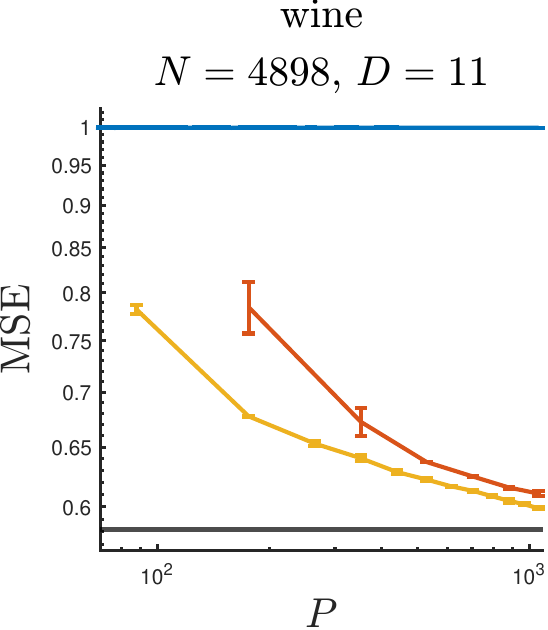}
    \end{subcaptionblock}%
    \begin{subcaptionblock}[T][][c]{.25\textwidth}
    \centering
    \includegraphics[width=\textwidth]{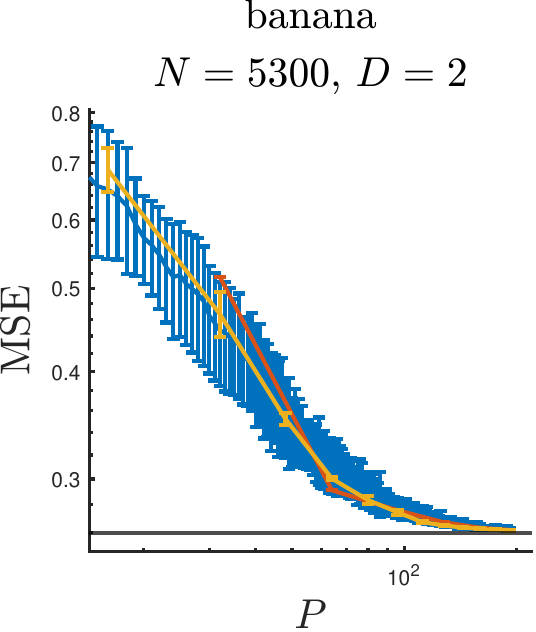}
    \end{subcaptionblock}%
    \caption{Plots of the train mean squared error as a function of the number of model parameters $P$, for different real-life datasets. In blue, random Fourier features \citep{rahimi_random_2007}, in red tensorized kernel machines with Fourier features \citep{wahls_learning_2014,stoudenmire_supervised_2016,kargas_supervised_2021,wesel_large-scale_2021}, in yellow quantized kernel machines with Fourier features, with quantization $Q=2$. The gray horizontal full line is the full unconstrained optimization problem, which corresponds to kernel ridge regression (KRR). The grey vertical dotted line is set at $P=N$. It can be seen that for $P<N$ case, quantization allows to achieve better performance with respect to the non-quantized case on the training set (this figure) \emph{and} on the test set (\cref{fig:generalization}).} 
    \label{fig:generalization_train}
\end{figure*}

\subsection{Regularizing Effect of Quantization}\label{appendix:sound}
In \cref{tbl:sound} we repeat the number of model parameters $P=2\log_2 M R$, the compression ratio of the quantized model weights $\nicefrac{M}{P}$, as well as the relative approximation error of the weights $\nicefrac{\vert\vert\mat{w}-\mat{w}_\text{CPD}\vert\vert}{\vert\vert\mat{w}\vert\vert}$ and the standardized mean absolute error (SMAE) of the reconstruction error on the test set as a function of the CPD rank.
\begin{table}[h]
\centering
\sisetup{scientific-notation = false,
        table-number-alignment=center,
        table-alignment-mode=none}
\caption{Model parameters, compression ratio and relative approximation error of the weights, and standardized mean absolute error on the test data as a function of the CPD rank.}
\label{tbl:sound}
\begin{tabular}{
S[table-number-alignment = center]
S[table-number-alignment = center]
S[table-format=2.1]
S[table-format=1.3]
S[table-format=1.3]
}
\toprule                           
 {$R$}                &  {$P$}            & {$\nicefrac{M}{P}$}     & {$\nicefrac{\vert\vert\mat{w}-\mat{w}_\text{CPD}\vert\vert}{\vert\vert\mat{w}\vert\vert}$}    & {SMAE}             \\
\cmidrule{1-5}
10              & 260   &  31.5  &  0.841  &  0.579  \\
25              & 650   &  12.6  &  0.712  &  0.571  \\
50              & 1300  &  6.3  &  0.528  &  0.451   \\
100             & 2600  & 3.1   &  0.310  &  0.182   \\
\bottomrule
\end{tabular}
\end{table}

\end{document}